\documentclass[pmlr]{jmlr} % new name PMLR (Proceedings of Machine Learning Research)

\newcommand{\tagdsmode}{proceedings}

\makeatletter
\newcommand{\tagdssubmission}{submission}
\newcommand{\tagdsproceedings}{proceedings}

\ifx\tagdsmode\tagdsproceedings
\else\ifx\tagdsmode\tagdssubmission
  \def\ps@jmlrtps{%
    \let\@mkboth\@gobbletwo
    \def\@oddhead{\scriptsize Under Review at the 2nd Conference on Topology, Algebra, and Geometry in Data Science\hfill}%
    \let\@evenhead\@oddhead
    \def\@oddfoot{}%
    \let\@evenfoot\@oddfoot
  }

\else
  \def\ps@jmlrtps{%
    \let\@mkboth\@gobbletwo
    \def\@oddhead{}%
    \let\@evenhead\@oddhead
    \def\@oddfoot{}%
    \let\@evenfoot\@oddfoot
  }
\fi\fi
\makeatother

\usepackage{longtable}% for long tables

\usepackage{booktabs}
\theorembodyfont{\upshape}
\theoremheaderfont{\scshape}
\theorempostheader{:}
\theoremsep{\newline}

\jmlrvolume{334}
\jmlryear{2026}
\jmlrworkshop{Topology, Algebra, and Geometry in Data Science}

\usepackage{amsfonts}
\usepackage{amssymb}
\usepackage{array}
\usepackage{textcomp}
\usepackage{stfloats}
\usepackage{url}
\usepackage{graphicx}
\usepackage{tikz}
\usetikzlibrary{arrows.meta,positioning,fit,backgrounds,calc,shapes.geometric}
\usepackage{float}
\usepackage{multirow}
\usepackage{booktabs}
\usepackage{amsmath}
\usepackage{mathrsfs}
\usepackage{xcolor}
\makeatletter
\let\c@subfigure\undefined\let\subfigure\undefined
\makeatother
\usepackage{subcaption}
\usepackage{nicefrac}
\usepackage{microtype}
\usepackage{mathtools}
\newcolumntype{P}[1]{>{\centering\arraybackslash}p{#1}}
\usepackage{algorithm}
\usepackage{algorithmic}
\usepackage{enumitem}
\usepackage{bm}
\usepackage[capitalize,noabbrev]{cleveref}

\usepackage{titlesec}
\titlespacing*{\section}{0pt}{1.2ex plus 0.3ex minus 0.2ex}{0.6ex plus 0.15ex}
\newcommand{\optfm}{\textsc{OptFM}}
\newcommand{\sgformer}{\textsc{SGFormer}}
\newcommand{\rwpe}{\textsc{RWPE}}
\newcommand{\rni}{\textsc{RNI}}
\newcommand{\lappe}{\textsc{LapPE}}

\newcommand{\multiset}[1]{\{\!\!\{#1\}\!\!\}}
\DeclareMathOperator*{\Pool}{Pool}
\DeclareMathOperator*{\HASH}{HASH}
\DeclareMathOperator*{\LinAttn}{LinAttn}

\usepackage[moderate,tracking=normal,paragraphs=normal]{savetrees}

\title[A 1-WL Characterization of MILP Graph Transformers]{A Weisfeiler--Leman Characterization of Global-Attention Graph Transformers for Mixed-Integer Linear Programs}

\ifx\tagdsmode\tagdssubmission

\else
\author{\Name{Md Abrar Jahin} \Email{jahin@usc.edu, jahin@isi.edu} \\
\Name{Craig A. Knoblock} \Email{knoblock@isi.edu}\\
\Name{Jay Pujara} \Email{jpujara@isi.edu}\\
\addr University of Southern California, Los Angeles, CA 90089, USA\\
\addr USC Information Sciences Institute, Marina del Rey, CA 90292, USA}

\fi

\ifx\tagdsmode\tagdsproceedings
\fi

\begin{document}

% \author{
% \IEEEauthorblockN{
% Md Abrar Jahin\IEEEauthorrefmark{1}
% }
% \IEEEauthorblockA{\IEEEauthorrefmark{1}\textit{\small Thomas Lord Department of Computer Science, University of Southern California, Los Angeles, CA 90089, USA}}
% \thanks{Corresponding author: Md Abrar Jahin (email: jahin@usc.edu)}
% }

%%% ORCID %%%%
% 0000-0002-1623-3859  <- Abrar
% 0009-0009-1894-9533 <- Akmol

% 0000-0001-8437-498X <- Nilanjan sir

% The paper headers
% \markboth{
% }{Jahin \MakeLowercase{\textit{et al.}}: DyCAF-Net}

\maketitle
\addtocounter{footnote}{-1}
\renewcommand{\thefootnote}{\fnsymbol{footnote}}
\footnotetext[2]{Code:~\scriptsize{\url{https://github.com/Abrar2652/optfm_probing/}}}
\renewcommand{\thefootnote}{\arabic{footnote}}

%% ================================================================
%%  ABSTRACT
%% ================================================================
\begin{abstract}
Graph foundation models (GFMs) equipped with global attention are increasingly used to learn representations of mixed-integer linear programs (MILPs), with the stated aim of capturing structural information beyond the locality of standard graph neural networks. We study the expressive power of these architectures through the lens of graph isomorphism testing and ask which MILP instances they map to identical representations. We prove that a broad class of hierarchical graph transformers combining global linear attention, edge-weighted cross-attention, and bipartite message passing is bounded by the one-dimensional Weisfeiler--Leman (1-WL) test: for any parameter setting, any pair of 1-WL-equivalent MILP graphs receives an identical graph embedding. The result follows from a compositional analysis in which each architectural component is shown to be a symmetric multiset function and therefore to preserve 1-WL equivalence. We validate the characterization across ten architecturally diverse graph encoders, including Graphormer-, GraphGPS-, Set-Transformer-, and Gasse-style models. Across model capacities, graph scales, and pooling operators, all tested encoders map 1-WL-equivalent non-isomorphic graph pairs to numerically identical embeddings. We then analyze the consequences of this representation equivalence for downstream prediction: graph invariants that vary within a 1-WL equivalence class are not recoverable from the resulting representations. Finally, we localize the source of expressiveness beyond 1-WL to the input encoding rather than the attention mechanism. Random-walk positional encodings separate the constructed pairs, while additional constructions characterize the limits of this remedy. Together, these results provide a theoretical and empirical characterization of the expressive power of global-attention graph foundation models, together with an encoder-agnostic diagnostic for detecting 1-WL-induced representation equivalence.
\end{abstract}


%% ================================================================
%%  1. INTRODUCTION
%% ================================================================
\section{Introduction}
\label{sec:intro}

Graph neural networks and, more recently, graph transformers have become central tools for learning over combinatorial optimization problems. Mixed-integer linear programs (MILPs) are a canonical target, represented as a bipartite graph with constraint nodes, variable nodes, and coefficient-weighted edges~\citep{gasse2019exact,prouvost2020ecole}. A growing line of work pre-trains graph foundation models (GFMs) on such representations and reuses the embeddings across downstream tasks such as branching, instance similarity, and feasibility or runtime prediction~\citep{yuan2025optfm,wu2023sgformer,liu2023towards,bengio2021machine}. A defining architectural feature of these models is global attention, which lets every node attend to every other node and is intended to model structure beyond the local neighborhoods accessed by message passing. Our analysis targets this architectural design rather than any specific pre-trained checkpoint: we study encoders that instantiate the same combination of global attention, cross-attention, and message passing at the scale accessible for exhaustive theoretical and empirical control, and our results are stated as properties of the architecture that hold for every parameter setting and therefore transfer to full-scale instances of the same design.

Existing studies of these architectures focus primarily on empirical performance; the theoretical understanding of what hierarchical graph-transformer encoders can represent is comparatively limited. In particular, it is not established how the combination of global attention, cross-attention between node types, and an auxiliary message-passing branch affects the set of graphs an encoder maps to distinct representations. Since a learned embedding is a bottleneck through which all downstream computation passes, two instances assigned the same embedding are indistinguishable to every predictor built on it, regardless of model size or training, so characterizing which instances share a representation is a prerequisite for understanding downstream behavior.

The standard tool for such a characterization is the one-dimensional Weisfeiler--Leman (1-WL) color-refinement test, which upper-bounds the distinguishing power of message-passing GNNs~\citep{cai1992lower,xu2019powerful,morris2019wlgo}. \citet{chen2023expressive} specialize this analysis to bipartite MILP graphs and exhibit 1-WL-equivalent instances that differ in feasibility. These results characterize message-passing architectures. Whether the global-attention mechanisms used in graph-transformer foundation models extend distinguishing power beyond 1-WL has not been resolved.

This paper addresses the following question: \emph{Does global attention increase expressive power beyond 1-WL for bipartite MILP graph encoders?} We study this question for a class of hierarchical encoders that combine global linear self-attention, edge-weighted cross-attention, a bipartite message-passing branch, a convex fusion step, and permutation-invariant pooling, matching the structure used by recent MILP graph foundation models. The theoretical contribution is a compositional expressiveness analysis: we formalize a 1-WL-bounded operator and show that each encoder stage is a symmetric function of a multiset of node representations. Since a composition of 1-WL-bounded operators is again 1-WL-bounded, the analysis yields a single characterization: for any parameter setting, the encoder assigns identical embeddings to 1-WL-equivalent MILP graphs. The empirical contribution validates this characterization and analyzes its downstream consequences. We construct families of 1-WL-equivalent non-isomorphic MILP graphs with a verified separation guarantee (\Cref{fig:hero}, \Cref{app:overview}), evaluate ten architecturally diverse encoders, probe what graph invariants the resulting embeddings encode, and examine input encodings above 1-WL that separate the constructed pairs.

Our main contributions are:
\textbf{(1)} We establish a 1-WL expressiveness characterization for a broad class of hierarchical graph-transformer architectures for MILP representations, showing that their constituent operations preserve 1-WL equivalence through symmetric multiset aggregation (\Cref{sec:theorem}). \textbf{(2)} We introduce a parameterized family of 1-WL-equivalent non-isomorphic MILP graph constructions and use it to evaluate ten graph encoders, including Graphormer-, GraphGPS-, Set-Transformer-, and Gasse-style architectures, across varying capacities, graph scales, and pooling operators (\Cref{sec:pairs,sec:main}). \textbf{(3)} We empirically characterize the consequences of representation equivalence through probing and downstream prediction experiments, and introduce an encoder-agnostic diagnostic for identifying 1-WL-induced representation equivalence in graph embeddings (\Cref{sec:probes,sec:discussion}). \textbf{(4)} We analyze the role of structural input encodings in extending distinguishability beyond the baseline representation, evaluating random-walk, Laplacian, and LP-derived features and identifying classes of graph pairs that remain indistinguishable under finite-length positional encodings (\Cref{sec:rwpe}). Because all evaluated architectures are permutation equivariant, each test pair is verified to be non-isomorphic as a feature-labeled graph, ensuring that indistinguishability cannot be attributed to label-preserving graph isomorphism (\Cref{sec:pairs}).

%% ================================================================
%%  2. RELATED WORK
%% ================================================================
\section{Related Work}
\label{sec:related}

\hspace{20pt}\textbf{Expressiveness and the Weisfeiler--Leman ladder.}
Expressive power is commonly measured by the number of non-isomorphic graphs a model distinguishes, with the 1-WL color-refinement test as the reference point~\citep{cai1992lower}. \citet{xu2019powerful} and \citet{morris2019wlgo} proved that no message-passing GNN separates two graphs that 1-WL declares equivalent, and a subsequent literature sought greater distinguishing power through higher-order networks, subgraph counting, and random features~\citep{morris2019wlgo,abboud2020rni,sato2021random}. The present analysis concerns where global-attention graph transformers fall relative to this hierarchy (\Cref{fig:hierarchy}).

\textbf{Expressiveness of graph transformers.}
A parallel line of work characterizes the distinguishing power of graph transformers themselves. \citet{zhu2023structural} introduce the SEG-WL test, an upper bound on a broad class of graph transformers, and show that their expressive power is governed by the choice of structural encoding rather than by attention. \citet{black2024comparing} prove that transformers with absolute and relative positional encodings are equivalent in distinguishing power on featureless graphs, isolating positional encoding as the locus of expressiveness. \citet{muller2024aligning} align pure transformers with the $k$-WL hierarchy through tokenization, obtaining transformers that provably exceed 1-WL via their input encoding. These results establish, for general graph transformers, that expressiveness is determined by structural or positional encodings. Our contribution is complementary and more specific: we prove an \emph{exact} 1-WL bound for the concrete hierarchical global-attention encoder used by MILP foundation models, covering global linear attention, edge-weighted cross-attention, a message-passing branch, convex fusion, and pooling under a single compositional argument, and we localize the residual expressiveness to the input encoding on bipartite MILP graphs, where feasibility itself can vary within a 1-WL class~\citep{chen2023expressive}.

\textbf{Graph transformers and ML for MILPs.}
Standard graph transformers include Graphormer~\citep{ying2021graphormer}, GraphGPS~\citep{rampasek2022gps}, \sgformer{}~\citep{wu2023sgformer}, and the Set Transformer~\citep{lee2019set}; these are among the architectures we analyze in \Cref{sec:setup}. The bipartite constraint--variable representation and solver-derived features originate with \citet{gasse2019exact}, are standardized in Ecole~\citep{prouvost2020ecole}, and underpin learning-to-branch and instance mining~\citep{bengio2021machine,khalil2017learning,gupta2020hybrid}. \citet{chen2023expressive} characterized GNN power for MILPs via 1-WL with a feasibility-differing pair, and \optfm{}~\citep{yuan2025optfm} is a recent multi-view graph-transformer foundation model for combinatorial optimization, which we treat as representative of the broader symmetric-aggregation design.

\textbf{Positional and structural encodings.}
Random-walk (RWPE) and Laplacian (LapPE) positional encodings inject structural signals that 1-WL cannot compute~\citep{dwivedi2022rwpe,dwivedi2021lap} and are strictly weaker than the 2-FWL test. Both separate the constructed pairs in our experiments (\Cref{sec:rwpe}). We report RWPE and LapPE together throughout; we foreground RWPE in the analysis of \Cref{sec:rwpe} only because its distinguishing action admits a closed-form explanation in terms of $k$-step return probabilities (\Cref{sec:rwpe}), which lets us derive the exact shortest distinguishing walk length and, conversely, construct pairs that finite-length RWPE provably cannot separate. LapPE separates the same pairs empirically but carries sign and basis ambiguities~\citep{dwivedi2021lap} that make an analogous closed-form characterization less direct; we use it, alongside RNI, as a corroborating control rather than as a claim that RWPE is uniquely effective.

%% ================================================================
%%  3. PRELIMINARIES
%% ================================================================
\section{Background and Setup}
\label{sec:prelim}
This section fixes the three ingredients the analysis rests on: the bipartite-graph representation of a MILP, the 1-WL color-refinement test on that graph, and the foundation-model encoder under study. Each is stated informally and then formally.

\textbf{From a MILP to a bipartite graph.} A MILP minimizes a linear objective under linear constraints with some integer variables: $\min\{c^\top x : Ax \circ b,\ l \le x \le u,\ x_j \in \mathbb{Z}\ (j \in I)\}$, $A \in \mathbb{R}^{m\times n}$, $\circ\in\{\le,=,\ge\}$ row-wise. The graph has one \emph{constraint node} per row and one \emph{variable node} per column, with an edge of weight $A_{ij}$ wherever $A_{ij}\neq 0$. Following Ecole~\citep{prouvost2020ecole}, variable nodes carry $d_v=9$ solver-derived features and constraint nodes $d_c=1$ (the normalized right-hand side).

\begin{definition}[Feature-labeled bipartite MILP]
\label{def:milp}
A feature-labeled bipartite MILP is a tuple $M=(C,V,A,f_C,f_V)$ with constraint nodes $C$, variable nodes $V$, coefficient matrix $A\in\mathbb{R}^{|C|\times|V|}$, and per-node feature maps $f_C:C\to\mathbb{R}^{d_c}$ and $f_V:V\to\mathbb{R}^{d_v}$.
\end{definition}

\textbf{The 1-WL color-refinement test.} The 1-WL test refines node colors iteratively. Every node begins with a color equal to its feature vector. At each step, a node updates its color based on its current color, the \emph{collection} of its neighbors' colors, and the edge weights connecting them. This collection records \emph{which} colors appear and \emph{how many} times each, but not their order; it is a multiset, denoted by $\multiset{\cdot}$. Formally,
\begin{equation}\label{eq:wl}
\begin{aligned}
\chi^{(0)}(u)&=f(u),\\
\chi^{(t+1)}(u)&=\HASH\!\big(\chi^{(t)}(u),\,\multiset{(\chi^{(t)}(w),A_{uw}):w\sim u}\big),
\end{aligned}
\end{equation}
where $w\sim u$ ranges over the neighbors of $u$. After the colors stop changing, we read off the histogram of colors for the whole graph.

\begin{definition}[1-WL equivalence]
\label{def:wleq}
Two feature-labeled bipartite MILPs $M_1,M_2$ are \emph{1-WL equivalent}, written $M_1\equiv_{\mathrm{1WL}}M_2$, if their global color histograms agree at every iteration $t\ge 0$.
\end{definition}

Isomorphic graphs are always 1-WL equivalent, but the converse fails~\citep{cai1992lower}. Our running example is \emph{one long cycle} versus \emph{several short cycles}: every node has two neighbors, so refinement assigns one color everywhere and identical histograms, yet the connected-component count differs. This count is invisible to 1-WL~\citep{xu2019powerful, chen2023expressive} and is consequential for an MILP, since a connected instance can couple decisions, whereas a disconnected one keeps them independent.

\textbf{The hierarchical foundation-model encoder.} The encoder maps a feature-labeled bipartite MILP to an embedding $z\in\mathbb{R}^h$ through six stages, matching recent MILP foundation models: \emph{(i)} per-type input projections $\tilde V_0=W_v f_V$, $\tilde C_0=W_c f_C$; \emph{(ii)} \sgformer{}-style normalized linear self-attention within each node type (Eq.~\ref{eq:linattn}); \emph{(iii)} cross-attention between types, concatenated with a dense edge-weighted message $A^\top\tilde C_1$ before a linear projection (Eq.~\ref{eq:cross}); \emph{(iv)} a parallel Gasse-style bipartite message-passing branch~\citep{gasse2019exact}; \emph{(v)} a per-node convex fuse $\tilde V=\alpha\tilde V^{(g)}+(1-\alpha)\tilde V_2$, $\alpha\in[0,1]$; and \emph{(vi)} permutation-invariant pooling $z=\Pool(\tilde C\oplus\tilde V)$, $\Pool\in\{\mathrm{mean},\mathrm{sum},\mathrm{max}\}$ (full operator definitions in \Cref{app:proofs}). During pre-training, two \emph{virtual global nodes} (side-wise feature means joined to every opposite-side node) are appended.

Two embeddings are termed \emph{bit-identical} when $\|\Phi(M_1)-\Phi(M_2)\|_\infty\le 10^{-5}$ in the model's working precision; \Cref{sec:main} shows the observed gaps fall nine orders of magnitude below this threshold.

%% ================================================================
%%  4. THEORY
%% ================================================================
\section{The Hierarchical Encoder Is 1-WL Bounded}
\label{sec:theorem}

This section establishes the main result: the encoder of \Cref{sec:prelim} assigns identical embeddings to any two MILPs that 1-WL declares equivalent, for any weight setting. The argument is compositional. Each stage of the encoder is a symmetric function of a multiset of node colors (\Cref{lem:l1,lem:l2,lem:l3,lem:l4,lem:l5}), and a composition of such functions preserves 1-WL equivalence.

\begin{theorem}[The hierarchical encoder is 1-WL bounded]
\label{thm:main}
Let $M_1,M_2$ be feature-labeled bipartite MILPs with
$M_1\equiv_{\mathrm{1WL}}M_2$. For any weights $\theta$ of the hierarchical encoder $\Phi_\theta$ and any pooling in $\{\mathrm{mean},\mathrm{sum},\mathrm{max}\}$,
\begin{equation}
\Phi_\theta(M_1)=\Phi_\theta(M_2)
\end{equation}
The conclusion still holds when both inputs are augmented with the virtual global nodes used during pre-training.
\end{theorem}

\paragraph{Proof strategy.}
The result does not follow from \citet{xu2019powerful} directly, since three components are not message-passing (global self-attention, cross-attention, virtual global nodes); each requires a separate argument, but all reduce to one mechanism. An operator is \emph{1-WL bounded} if, whenever its inputs agree as multisets of colors (a color-preserving bijection matches the nodes), its outputs agree as multisets under that bijection. 1-WL equivalence (\Cref{def:wleq}) supplies the input bijection, and a composition of 1-WL-bounded operators is 1-WL bounded, so it suffices to verify one component at a time.

\begin{definition}[1-WL bounded operator]
\label{def:bounded}
An operator $\Phi$ on node embeddings is \emph{1-WL bounded} if, whenever two input node sets agree as multisets (a bijection matches their colors), the output node sets agree as multisets under the same bijection.
\end{definition}

We establish one lemma per component (\Cref{lem:l1,lem:l2,lem:l3,lem:l4,lem:l5}; full statements and proofs in \Cref{app:proofs}): \textbf{(L1)} \sgformer{} linear self-attention \eqref{eq:linattn} depends on the source set only through the global sums $\sum_\ell k_\ell^\top v_\ell$ and $\sum_\ell k_\ell^\top\mathbf{1}$, both multiset functions; \textbf{(L2)} cross-attention with its edge message $m_i=\sum_j A_{ij}y_j$ \eqref{eq:cross} is a sum over the edge-weighted neighbourhood multiset that 1-WL refinement already matches; \textbf{(L3)} the Gasse-style GCN branch is an MPNN, hence 1-WL bounded by \citet{xu2019powerful} as specialized to MILPs by \citet{chen2023expressive}; \textbf{(L4)} the per-node convex fuse and mean/sum/max pooling are order-independent functions of the node multiset; and \textbf{(L5)} the two virtual global nodes add side-wise feature means under a fixed uniform edge pattern, preserving equivalence. Composing these, \Cref{thm:main} follows: the input bijection from \Cref{def:wleq} is preserved through stages (i)--(vi), so the pooled embeddings are equal.\hfill$\square$

% \paragraph{Consequences.}
\Cref{thm:main} holds for every weight setting, so the characterization is a property of the architecture rather than of a trained model. Because every stage is 1-WL bounded for the same reason, symmetric multiset aggregation, any encoder composed of per-node maps and symmetric aggregation admits the same characterization. A per-component ablation (\Cref{tab:lemma}, \Cref{app:proofs}) and a sub-layer analysis (\Cref{fig:layerwise}, \Cref{app:robust}) confirm that representation equivalence holds at every stage, while the positional-encoding control does not.

%% ================================================================
%%  5. PAIRS
%% ================================================================
\section{Construction of 1-WL-Equivalent Test Pairs}
\label{sec:pairs}
Testing \Cref{thm:main} requires pairs of MILPs that are non-isomorphic yet 1-WL equivalent. Constructing such pairs correctly is nontrivial, and an undetected error invalidates the experiment. This section gives a construction with a separation guarantee and the validity checks that every pair must pass.

% \paragraph{Validity check.}
Each pair is verified non-isomorphic \emph{as a feature-labeled bipartite graph}: since every encoder is permutation-equivariant, a matching-embeddings result on a secretly isomorphic pair would confirm only equivariance, not the 1-WL bound. We certify non-isomorphism by exact subgraph isomorphism~\citep{cordella2004sub} on small graphs and the (1-WL-invisible) connected-component count on large ones, and confirm each member's feasibility with an exact solver. For degree $d$ and scale $k$: $G_A$ is a \emph{connected} $d$-regular bipartite circulant on $kd{+}kd$ nodes (constraint $i$ joins variables $i,\dots,i{+}d{-}1\bmod kd$), while $G_B$ is $k$ disjoint copies of $K_{d,d}$. Both are $d$-regular with identical uniform features, so 1-WL refines each to a single color class and declares them equivalent; yet $G_A$ has one connected component and $G_B$ has $k$, so they are non-isomorphic. Component count is invisible to 1-WL, which is what makes this a valid test. We use two instances (\Cref{fig:pairs}): the cycle family ($d=2$: $C_{4k}$ vs.\ $k\cdot C_4$) and the cubic family ($d=3$ vs.\ $2\cdot K_{3,3}$); the latter controls for a cycle-specific artifact, having different regularity, girth, and refinement mechanism while exhibiting the same representation equivalence. Re-instantiating the same topology under standard MILP feature distributions (set cover, combinatorial auctions, bipartite matching, multi-knapsack blocks), and under heterogeneous and gadget-union variants carrying integer variables, finite bounds, non-zero objectives, and node-varying features, still yields bit-identical embeddings across all six encoders, controlling for the zero-feature case (\Cref{tab:natural}, \Cref{app:robust}). Node-varying features are assigned so that 1-WL equivalence is preserved by construction: features are drawn from a fixed alphabet and placed identically along the two topologies under the color-preserving bijection guaranteed by 1-WL equivalence, so that constraint~$i$ of $G_A$ and its image in $G_B$ receive the same feature value. Concretely, since 1-WL refines both $G_A$ and $G_B$ to a single color class (all nodes on a side are structurally interchangeable), any assignment that is constant within a side, or any assignment transported across the bijection, leaves the two color histograms identical at every refinement step and therefore keeps the pair 1-WL equivalent while making the features non-trivial. We verify equivalence directly by recomputing the 1-WL color histograms after feature assignment for every pair (a check that fails immediately if a feature placement breaks the symmetry). WL-twin nodes are moreover common in random instances (exact rates in \Cref{fig:prevalence}, \Cref{app:figs}), so the characterization applies whenever an instance carries a regular or symmetric substructure, which standard families routinely do.

\section{Empirical Verification}
\label{sec:main}
\label{sec:setup}
We evaluate ten graph encoders spanning both the analyzed foundation-model family and independently developed architectures. The foundation-model group includes the pre-trained \sgformer{}+GCN checkpoint, its randomly initialized counterpart, attention-only and GCN-only ablations, a minimal GCN, and the full hierarchical model. The independent baselines comprise Graphormer-~\citep{ying2021graphormer}, GraphGPS-~\citep{rampasek2022gps}, Set Transformer-~\citep{lee2019set}, and the bipartite GCN of~\citet{gasse2019exact}. Each encoder is evaluated under six input configurations: the baseline encoding, virtual-global-node augmentation, \rni{}~\citep{abboud2020rni}, LP-derived features, \rwpe{}~\citep{dwivedi2022rwpe}, and \lappe{}~\citep{dwivedi2021lap}. Complete implementation details, transform definitions, and reproducibility information are provided in \Cref{app:expdetails,app:repro}.

Our primary evaluation uses the 15 graph pairs introduced in \Cref{sec:pairs}, consisting of fourteen cycle-family instances and one cubic-graph pair. For each encoder and graph pair, we measure cosine similarity, exact-match rate (EM), defined as the fraction of pairs satisfying $|\Phi(M_A)-\Phi(M_B)|_\infty \leq 10^{-5}$, and the maximum coordinate-wise embedding difference. Experiments are repeated across five random seeds. To assess robustness, we additionally vary model capacity (width 16--256, depth 1--8, and 1--4 attention heads), graph size (up to 800 nodes), pooling strategy (mean, sum, and max), and numerical precision (float32 and float64). As specificity controls, we evaluate the same encoders on 1-WL-distinguishable graph pairs and on the heterogeneous and gadget-union constructions described in \Cref{sec:pairs}.

\begin{table}[t]
\centering
\caption{Bit-identical embeddings across ten architectures at baseline (mean $\pm$ std over five seeds). EM is the exact-match rate, the fraction of 15 pairs with $\|\Phi(M_A)-\Phi(M_B)\|_\infty\le10^{-5}$; ``worst $\|\cdot\|_\infty$'' is the largest coordinate gap over all pairs and seeds (float32). Top block: the foundation-model family; bottom block: independently designed encoders.}
\label{tab:main}
\small
\begin{tabular}{@{}lrccc@{}}
\toprule
Encoder & params & mean cos & worst $\|\cdot\|_\infty$ & EM \\
\midrule
\sgformer{}+GCN (pre-trained) & 6{,}820 & $1.000000\pm0$ & $6.0\times10^{-8}$ & 1.00 \\
\sgformer{}+GCN (random) & 6{,}820 & $1.000000\pm0$ & $6.0\times10^{-8}$ & 1.00 \\
attention-only & 1{,}344 & $1.000000\pm0$ & $3.7\times10^{-8}$ & 1.00 \\
GCN-only & 5{,}602 & $1.000000\pm0$ & $4.5\times10^{-8}$ & 1.00 \\
minimal GCN & 736 & $1.000000\pm0$ & $2.1\times10^{-8}$ & 1.00 \\
hierarchical (full) & 10{,}852 & $1.000000\pm0$ & $6.0\times10^{-8}$ & 1.00 \\
\midrule
Graphormer-style & 21{,}616 & $1.000000\pm0$ & $5.4\times10^{-8}$ & 1.00 \\
GraphGPS-style & 21{,}568 & $1.000000\pm0$ & $5.1\times10^{-8}$ & 1.00 \\
Set-Transformer pooling & 13{,}216 & $1.000000\pm0$ & $4.9\times10^{-8}$ & 1.00 \\
Gasse-2019 bipartite GCN & 25{,}600 & $1.000000\pm0$ & $5.96\times10^{-8}$ & 1.00 \\
\bottomrule
\end{tabular}
\end{table}

Across all ten encoders and all 15 graph pairs, the resulting graph embeddings are identical up to floating-point precision. Every architecture achieves a mean cosine similarity of $1.000000$ (standard deviation below float32 resolution across five seeds) and an exact-match rate of $1.00$ (\Cref{tab:main}). The largest coordinate-wise discrepancy observed is $5.96\times10^{-8}$ in float32 and $5.6\times10^{-17}$ in float64, with the latter corresponding to machine-precision numerical error and remaining several orders of magnitude below the matching threshold (\Cref{fig:identity}; \Cref{app:robust}). The observed equivalence remains unchanged across all capacity settings, graph scales, and pooling operators. Increasing model size from approximately $10^4$ to $2.5\times10^6$ parameters does not alter the result, nor does scaling the graph instances to 800 nodes (\Cref{fig:capacity}; \Cref{app:robust}). In contrast, structural input encodings such as \rwpe{} consistently produce distinguishable representations, providing a positive control for the evaluation procedure. The specificity controls further demonstrate that the encoders do not collapse all inputs to a common representation. For graph pairs that are distinguishable by 1-WL, all architectures produce distinct embeddings. Likewise, the heterogeneous and gadget-union constructions remain consistent with the theoretical predictions. Together, these results provide empirical support for the representational equivalence characterized by \Cref{thm:main} across a broad range of architectures, capacities, and implementation choices. \Cref{fig:main} (\Cref{app:figs}) shows the full encoder$\times$transform grid underlying \Cref{tab:main}.

%% ================================================================
%%  8. PROBES
%% ================================================================
\section{Probing Information Content of the Embeddings}
\label{sec:probes}
To characterize the information retained by the learned representations, we freeze each encoder and train both linear and MLP probes to predict graph-level invariants from the resulting embeddings. We evaluate two populations: (i) 300 randomly generated bipartite MILPs, which serve as a positive control, and (ii) a population of 1-WL-equivalent instances. The prediction targets include the number of connected components, the number of $4$-cycles, and the spectral gap. As a control for probe capacity, we repeat the analysis using \rwpe{}-augmented embeddings of the same instances. The results reveal a consistent difference between the two populations. For the connected-components target, probe performance reaches $R^2=0.23\pm0.04$ on random MILPs but drops to $R^2=-0.01\pm0.02$ on the 1-WL-equivalent population (mean $\pm$ std over five seeds), despite substantially larger target variance in the latter ($26.37$ versus $0.22$). Similar behavior is observed for $4$-cycle counts and spectral-gap prediction. When the same probes are applied to \rwpe{}-augmented embeddings, performance improves substantially; for example, the connected-components probe reaches $R^2=0.51\pm0.05$. Complete results, target variances, and control experiments are reported in \Cref{tab:probes} (\Cref{app:probes}), with the full probe battery shown in \Cref{fig:probes}. These findings indicate that graph invariants varying within a 1-WL equivalence class are not reliably recoverable from the baseline representations, even when the target exhibits substantial variation. In contrast, representations augmented with structural positional information retain sufficient information for downstream recovery of the same targets. The observed pattern is consistent across both linear and nonlinear probes and aligns with the expressiveness characterization established in \Cref{thm:main}.

%% ================================================================
%%  9. RWPE
%% ================================================================
\section{Localizing the Source of Expressiveness}
\label{sec:rwpe}
To determine whether the observed representation equivalence can be overcome through optimization alone, we jointly train the full encoder and a binary classification head to distinguish $G_A$ from $G_B$ across the 15 constructed pairs. We evaluate the baseline model together with four input augmentations: \rni{}, \rwpe{}, \lappe{}, and LP-derived features. In addition, we analyze the effect of the \rwpe{} walk length and search for 1-WL-equivalent non-isomorphic graph pairs that remain indistinguishable under \rwpe{}. Under the baseline encoding, training converges to the binary cross-entropy optimum for random guessing, with a final loss of $0.69314748 \pm 4.8 \times 10^{-7}$ across five seeds. The resulting graph embeddings remain bit-identical after optimization (\Cref{fig:plateau}; \Cref{app:figs}), indicating that the constructed pairs are not separable through training alone. In contrast, \rwpe{} separates the embeddings for all evaluated encoders and reduces the classification loss to approximately $8 \times 10^{-4}$.

The separation induced by \rwpe{} is consistent with the underlying random-walk statistics of the constructed graphs. Specifically, the $4$-step return probability differs between $C_{4k}$ and $k \cdot C_4$ ($P^4[i,i]=3/8$ versus $1/2$), yielding a shortest distinguishing walk length of $L=4$. A walk-length sweep confirms this prediction: mean cosine similarity remains $1.0000$ for $L \in {2,3}$, decreases to $0.984$ at $L=4$, and further decreases to $0.940$ at $L=16$. Similar separation is obtained with \rni{} and \lappe{}, whereas LP-derived features leave most pairs indistinguishable (\Cref{tab:enc}; \Cref{app:robust}; \Cref{fig:encoding}; \Cref{app:figs}). The effectiveness of \rwpe{} is nevertheless limited. We identify 1-WL-equivalent non-isomorphic graph pairs that remain indistinguishable under finite-length random-walk encodings. For example, a pair with $m=10$ and $d=4$ (offset sets ${0,1,2,5}$ and ${0,1,3,5}$) is not separated by walk lengths ${4,6,8}$, and a second pair with $m=12$ and $d=4$ remains indistinguishable even when all even walk lengths up to $16$ are included. These examples are cospectral with respect to the corresponding random-walk statistics, indicating that finite-length \rwpe{} does not eliminate all 1-WL equivalences.

%% ================================================================
%%  10. DISCUSSION
%% ================================================================
\section{Downstream Consequences and Diagnostic}
\label{sec:consequences}
The representational invariance established by Theorem~\ref{thm:main} has direct implications for downstream predictors operating on graph embeddings. To evaluate these implications, we consider a connectivity-classification task in which the target depends on the number of connected components and therefore varies within a 1-WL equivalence class. We train embedding-only classifiers using the representations produced by each of the six foundation-family encoders under both the baseline input encoding and the \rwpe{} augmentation. In addition, we quantify the prevalence of WL-twin nodes in randomly generated instances from standard MILP families. Under the baseline encoding, connectivity prediction remains near chance, with accuracies ranging from $0.34\pm0.03$ to $0.39\pm0.04$ across all encoders (mean $\pm$ std over five seeds). Introducing \rwpe{} substantially improves performance, yielding near-perfect accuracy ($0.97$--$1.00$) for five of the six architectures; the exception is the GCN-only ablation, which does not effectively propagate the positional signals to the pooled representation. Figure~\ref{fig:footprint} (\Cref{app:figs}) summarizes these results together with the separation induced by different input encodings. The degeneracy is not restricted to the constructed pairs: across random instances from standard MILP families, WL-twin nodes occur in roughly $6\%$--$12\%$ of cases.

\section{Discussion}
\label{sec:discussion}
The theoretical and empirical results jointly indicate that the hierarchical graph encoders studied here do not exceed 1-WL on the baseline MILP representation. The equivalence is invariant to parameters, capacity, and optimization, indicating that it arises from the combination of per-node transformations and symmetric aggregation rather than from a particular training configuration. 

Several graph properties relevant to mixed-integer programming, including connectivity, cycle structure, and other invariants identified in the MILP-learning literature~\citep{chen2023expressive}, may vary within a 1-WL equivalence class, and the consequence is concrete. \citet{chen2023expressive} exhibit two bipartite MILPs that 1-WL cannot separate yet that differ in \emph{feasibility}; any 1-WL-bounded encoder assigns them the same embedding, so a feasibility classifier built on it cannot be correct on both, regardless of data or capacity. The same holds for any non-1-WL-invariant target: connected-component count governs whether an instance decomposes into independent subproblems, and constraint-graph cycle structure relates to LP-relaxation strength. When such a quantity is the target or a needed feature, a collapsed embedding caps accuracy at that of a predictor ignoring the distinguishing structure; the probing and downstream experiments make this cap visible, while a control that adds a distinguishing signal recovers the targets. Structural positional encodings are one such mechanism: random-walk and Laplacian encodings separate the constructed pairs, whereas LP-derived features generally do not. Gains attributed to graph-transformer architectures should therefore be read together with the expressive contribution of the input representation. The characterization covers architectures built from per-node transformations and symmetric multiset aggregation, including the hierarchical encoder studied here and several widely used graph-transformer variants; architectures that add shortest-path features, higher-order message passing, or other non-local mechanisms fall outside this characterization. The empirical study targets exact 1-WL-equivalent constructions; extending it to approximate symmetries and large-scale solver-labeled benchmarks is left to future work (\Cref{app:threats}).

\section{Conclusion}
\label{sec:conclusion}
We characterized a class of hierarchical global-attention graph encoders for MILPs as 1-WL bounded: for every parameter setting, they assign identical embeddings to 1-WL-equivalent instances, verified across ten encoders and several graph families, with practical expressiveness on these instances set by the input representation rather than the attention mechanism. We release an encoder-agnostic diagnostic for detecting 1-WL-induced representation equivalence in graph embeddings.

\bibliography{main}

\appendix

\section{Overview and Construction Figures}
\label{app:overview}

\begin{figure}[t]
\centering
\includegraphics[width=0.8\columnwidth]{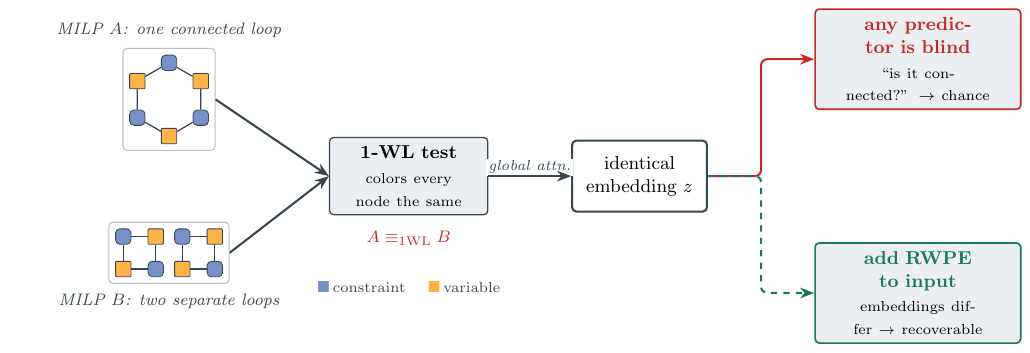}
\caption{Schematic of the main characterization. Two non-isomorphic MILPs, one whose bipartite graph is a single connected component ($A$) and one that decomposes into separate components ($B$), are 1-WL equivalent because every node has the same refined color. A hierarchical global-attention encoder maps $A$ and $B$ to the same embedding $z$ for any parameter setting (\Cref{thm:main}). Any predictor that consumes only $z$ therefore assigns the same output to both, including for invariants such as connectivity that differ between the two. The additional expressiveness required to separate them is provided by the input encoding rather than the attention mechanism: a random-walk positional encoding (RWPE) distinguishes the two embeddings.}
\label{fig:hero}
\end{figure}

\Cref{fig:hero} provides a schematic overview of the main characterization: two 1-WL-equivalent non-isomorphic MILP graphs receive the same embedding, any predictor consuming that embedding assigns them the same output, and a random-walk positional encoding separates them. \Cref{fig:pairs} shows the two 1-WL-equivalent test families used in \Cref{sec:pairs}.

\begin{figure}[t]
\centering
\includegraphics[width=0.7\textwidth]{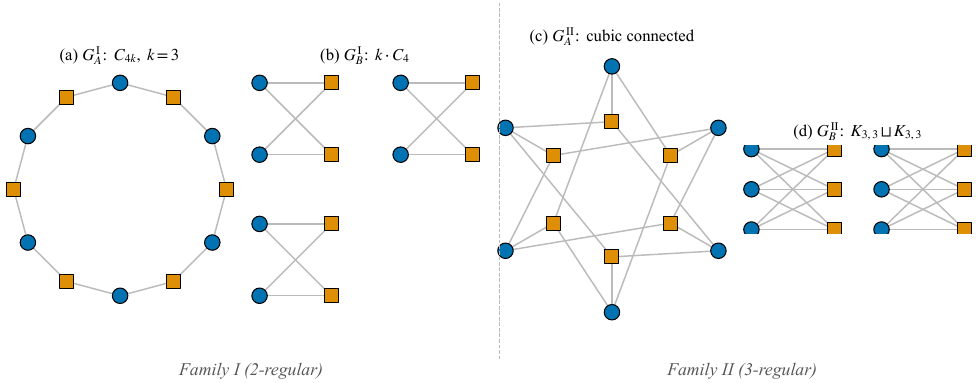}
\caption{The two 1-WL-equivalent, non-isomorphic test families of \Cref{sec:pairs}. \emph{Family I (cycles):} one long cycle $C_{4k}$ versus $k$ short cycles $k\cdot C_4$. \emph{Family II (cubic):} a connected $3$-regular bipartite graph versus $2\cdot K_{3,3}$. Both graphs in each pair are regular and have identical features, so 1-WL assigns a single color to every vertex, yet the pair differs in the number of connected components.}
\label{fig:pairs}
\end{figure}

\section{Reproducibility and Robustness}
\label{app:repro}

All experiments are completed within minutes on a laptop CPU and require neither GPUs nor cluster computing. We provide a public code artifact containing the seeded, automated pipeline used to produce all reported results. For each graph pair, the accompanying test suite verifies equality of node and edge counts, degree sequences, and 1-WL color histograms, and confirms non-isomorphism and the intended difference in connected-component structure. Non-isomorphism is verified using an exact graph-isomorphism test for small instances and the differing component counts for large instances. The test suite additionally checks permutation equivariance of the encoder, deterministic execution, reproducibility of the bit-identity result, and the specificity control described in \Cref{sec:main}. The principal results are stable across five random seeds: without \rwpe{}, the fine-tuning cross-entropy remains equal to $\ln 2$ to seven decimal places, whereas with \rwpe{} it decreases to approximately $8\times10^{-4}$. The bit-identity result holds up to floating-point precision in both float32 (maximum deviation $\leq 5.96\times10^{-8}$) and float64 ($\leq 5.6\times10^{-17}$), including for graphs with up to 800 nodes.

\section{Lemmas and Detailed Proofs}
\label{app:proofs}
This appendix states the five lemmas underlying \Cref{thm:main} with full proofs, and reports the component ablation (\Cref{tab:lemma}) that matches each lemma to an empirical test.

\paragraph{Encoder attention operations.}
The intra-type self-attention of stage (ii) is \sgformer{}-style normalized linear attention: for a query set $Q$ and source set $X$,
\begin{equation}\label{eq:linattn}
\begin{aligned}
\LinAttn(Q,X)_i &= D_i^{-1}\!\Big(V_i+\tfrac{1}{N}\,\tilde Q_i\,(\tilde K^\top V)\Big),\\
D_i &= 1+\tfrac{1}{N}\,\tilde Q_i\,(\tilde K^\top\mathbf{1}),
\end{aligned}
\end{equation}
with Frobenius-normalized $\tilde Q=W_qQ/\|W_qQ\|_F$ and $\tilde K=K/\|K\|_F$, computed in $O(N)$ time via the global statistics $\tilde K^\top V$ and $\tilde K^\top\mathbf{1}$. The cross-attention of stage (iii) concatenates the attention output with a dense edge-weighted message before a linear projection; for the constraint-to-variable pass,
\begin{equation}\label{eq:cross}
\tilde V_2 = W_o\big[\,\LinAttn(\tilde V_1,\tilde C_1);\ A^\top\tilde C_1\,\big],
\end{equation}
and symmetrically the other way.

\begin{corollary}[Indistinguishable instance families]
\label{cor:concrete}
The encoder cannot distinguish the connected bipartite cycle $C_{4k}$ from $k$ disjoint $4$-cycles $k\cdot C_4$ for any $k\ge2$, nor the connected $3$-regular bipartite graph on $6{+}6$ nodes from two disjoint copies of $K_{3,3}$, although each pair is non-isomorphic and differs in the number of connected components.
\end{corollary}

\begin{table}[t]
\centering
\caption{Each lemma maps to an ablation that isolates one component of the encoder. The exact-match rate (EM) is the fraction of the 15 test pairs mapped to bit-identical embeddings. All components are bit-identical, matching the theory.}
\label{tab:lemma}
\small
\begin{tabular}{@{}llc@{}}
\toprule
Lemma & Component isolated by the ablation & EM \\
\midrule
\Cref{lem:l1} & linear self-attention (attention-only) & 1.00 \\
\Cref{lem:l2} & cross-attention + edge messages (full) & 1.00 \\
\Cref{lem:l3} & bipartite GCN branch (GCN-only) & 1.00 \\
\Cref{lem:l4} & fuse + pooling (mean/sum/max) & 1.00 \\
\Cref{lem:l5} & virtual-global-node augmentation & 1.00 \\
\bottomrule
\end{tabular}
\end{table}

\begin{lemma}[Linear self-attention is a multiset function]
\label{lem:l1}
\sgformer{} normalized linear attention at node $i$ in \eqref{eq:linattn} depends only on the pair $(q_i,v_i)$ at that node and on the two global sums $\sum_\ell k_\ell^\top v_\ell$ and $\sum_\ell k_\ell^\top\mathbf{1}$. Both sums are order-independent functions of the multiset $\multiset{(k_\ell,v_\ell)}$ of the source set, and the Frobenius normalization depends only on $\|Q\|_F$, itself a multiset function. Hence, linear attention is 1-WL bounded over the query set.
\end{lemma}

\begin{proof}
The factored numerator $\tilde Q_i(\tilde K^\top V)$ and denominator $\tilde Q_i(\tilde K^\top\mathbf{1})$ contain the source set only through the sums $\sum_\ell k_\ell^\top v_\ell$ and $\sum_\ell k_\ell^\top\mathbf{1}$, which are invariant to any reordering of the source nodes. The only node-specific dependence is through $(q_i,v_i)$. If two query sets agree as multisets and their source sets agree as multisets, then matched query nodes produce matched outputs.
\end{proof}

\begin{lemma}[Cross-attention and edge messages preserve equivalence]
\label{lem:l2}
The cross-attention block \eqref{eq:cross} is 1-WL bounded. The added edge message $m_i=\sum_j A_{ij}\,y_j$ is an order-independent function of the edge-weighted neighborhood. Since 1-WL equivalence supplies a bijection preserving the weight--color pairs $(A_{ij},\chi(j))$, the message $m_i$ matches across the two graphs, and the concatenation and per-node projection preserve the match.
\end{lemma}

\begin{proof}
The attention part is 1-WL bounded by \Cref{lem:l1}. For the edge term, the very definition of 1-WL refinement \eqref{eq:wl} hashes the multiset $\multiset{(\chi(j),A_{ij})}$; two 1-WL-equivalent nodes therefore have a bijection between their incident edges preserving both the weight $A_{ij}$ and the neighbor color $\chi(j)$. Because $m_i=\sum_j A_{ij}y_j$ is a sum over that multiset, the two nodes get the same $m_i$. Concatenation and the linear projection $W_o$ act per node, so they preserve the match.
\end{proof}

\begin{lemma}[The bipartite GCN branch is 1-WL bounded]
\label{lem:l3}
The Gasse-style bipartite message-passing branch is a message-passing neural network (MPNN), and is therefore 1-WL bounded by the results of Xu et al.~\citep{xu2019powerful} as specialized to bipartite MILPs by Chen et al.~\citep{chen2023expressive}.
\end{lemma}

\begin{lemma}[Fusion and pooling preserve equivalence]
\label{lem:l4}
The per-node convex combination $\alpha\tilde v^{(g)}+(1-\alpha)\tilde v$ acts on aligned nodes and so preserves the multiset match, and mean, sum, and max pooling are
all order-independent functions of the final node multiset. Hence, two graphs whose node multisets agree receive equal pooled embeddings.
\end{lemma}

\begin{lemma}[Virtual global nodes do not help]
\label{lem:l5}
Appending the two virtual global nodes preserves 1-WL equivalence. Each global node's feature is the mean of one side's features, an order-independent function of an already-matched multiset, and the edges it adds follow a fixed, uniform pattern that is identical on both graphs. The augmented graphs are therefore again 1-WL equivalent.
\end{lemma}

\section{Additional Robustness Results}
\label{app:robust}
This appendix collects the tables and figures that support the empirical claims of \Cref{sec:main,sec:pairs,sec:rwpe}. \Cref{tab:main} gives the headline result: all ten encoders of \Cref{sec:setup}, spanning the foundation-model family and four independent designs, map every one of the fifteen test pairs to a bit-identical embedding (mean cosine $1.000000$, exact-match rate $1.00$); on the hierarchical model, the diagnostic of \Cref{sec:discussion} accordingly returns a ``1-WL bounded'' verdict at EM $1.00$. \Cref{tab:natural} repeats the experiment with the synthetic topology dressed in the feature distributions of four standard MILP families, confirming that the representation equivalence reported in \Cref{sec:pairs} is not an artifact of the bare construction. \Cref{tab:enc} summarizes the input-encoding comparison underlying \Cref{sec:rwpe}: the baseline and LP-derived features leave the bound intact, whereas \rni{}, \rwpe{}, and \lappe{} separate the pairs. \Cref{fig:hierarchy} places every model on the Weisfeiler--Leman expressiveness ladder, placing the global-attention encoder at the same 1-WL level as a plain message-passing GNN, with the structural encodings strictly above it. \Cref{fig:layerwise} traces the embedding gap through the encoder's nine sub-layers, verifying that the representation equivalence of \Cref{thm:main} holds at every intermediate representation rather than only after pooling. \Cref{fig:identity} audits the numerical precision behind the ``bit-identical'' verdict of \Cref{sec:main}, and \Cref{fig:capacity} shows the equivalence is invariant as the model grows from $10{,}852$ to $2{,}508{,}292$ parameters. For reference, the full hierarchical model uses $\sim\!14.7\times$ the parameters and $\sim\!22\times$ the forward time of a plain GCN ($1.75$ vs.\ $0.08$\,ms) while producing the same embedding on the test pairs.

\begin{table}[t]
\centering
\caption{The bound persists when the pairs are dressed in the feature distributions of standard MILP families. Each family is instantiated at seven scales ($k=2,\dots,10$, giving $G_B$ up to ten disconnected components) and probed with all six foundation-family encoders ($7\times6=42$ cells). Every cell is bit-identical; the worst-case coordinate gap equals single-precision machine epsilon.}
\label{tab:natural}
\small
\begin{tabular}{@{}lccc@{}}
\toprule
MILP family & bit-identical & worst $\|\cdot\|_\infty$ & EM \\
\midrule
set cover & 42 / 42 & $6.0\times10^{-8}$ & 1.00 \\
combinatorial auctions & 42 / 42 & $6.0\times10^{-8}$ & 1.00 \\
bipartite matching & 42 / 42 & $6.0\times10^{-8}$ & 1.00 \\
multi-knapsack blocks & 42 / 42 & $6.0\times10^{-8}$ & 1.00 \\
\bottomrule
\end{tabular}
\end{table}

\begin{table}[t]
\centering
\caption{Input-encoding comparison. EM is the fraction of pairs left bit-identical; ``min cos'' is the smallest mean cosine over encoders (lower means a cleaner break).}
\label{tab:enc}
\small
\begin{tabular}{@{}lccl@{}}
\toprule
Encoding & EM & min cos & note \\
\midrule
baseline & 1.00 & 1.0000 & the 1-WL bound \\
\rni{} & 0.00 & 0.9950 & stochastic; no guarantee \\
\rwpe{} & 0.00 & 0.9968 & deterministic; $<$ 2-FWL \\
\lappe{} & 0.00 & 0.9992 & sign/basis ambiguity \\
LP-derived & 0.94 & 1.0000 & weak; LP optimum coincides \\
\bottomrule
\end{tabular}
\end{table}

\begin{figure}[t]
\centering
\includegraphics[width=0.7\columnwidth]{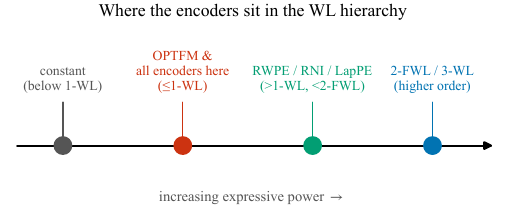}
\caption{Where the models in this paper sit on the expressiveness ladder. Every encoder we study, including the global-attention foundation model, lands at or below the 1-WL rung, the same ceiling as a plain message-passing GNN. The positional encodings we use as a remedy (RWPE, RNI, LapPE) sit strictly above 1-WL but below the second-order folklore test (2-FWL), which is why they break our specific pairs without being a universal fix.}
\label{fig:hierarchy}
\end{figure}

\begin{figure}[t]
\centering
\includegraphics[width=0.7\columnwidth]{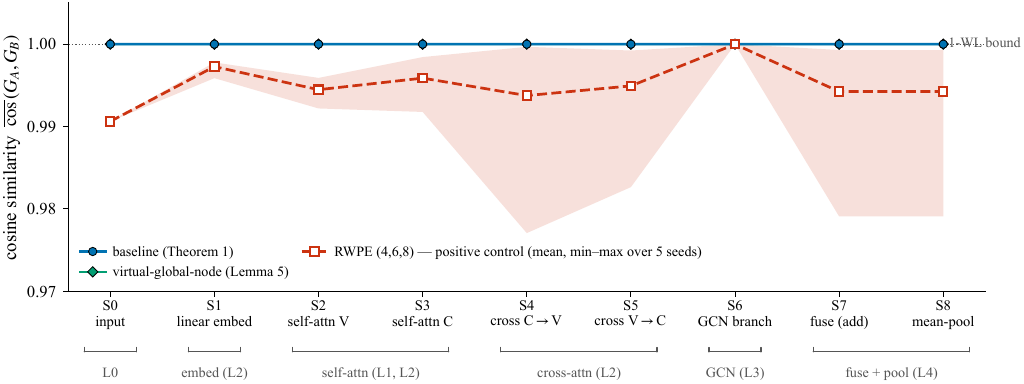}
\caption{Layer-by-layer view of the bound inside the hierarchical encoder. Moving left to right through the nine sub-layers (S0 input through S8 mean-pool, annotated with the lemma each one falls under), the baseline cosine between the two embeddings stays pinned at $1.0$ at every stage, and the virtual-global-node variant (\Cref{lem:l5}) coincides with it. The \rwpe{} control (mean and min--max band over five seeds) departs
from $1.0$ from the first layer onward. The bound is therefore not an artifact of the final pooling step; it holds at every intermediate representation, exactly as the compositional proof predicts.}
\label{fig:layerwise}
\end{figure}

\begin{figure}[t]
\centering
\includegraphics[width=0.7\columnwidth]{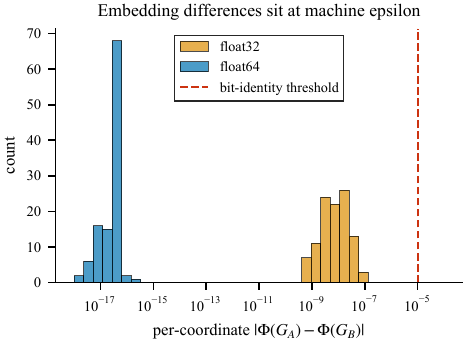}
\caption{Distribution of the per-coordinate embedding gap $|\Phi(M_A)-\Phi(M_B)|$ over all pairs and encoders. In double precision (float64) the gaps cluster near $10^{-17}$; in single precision (float32) near $10^{-8}$. Both populations sit far to the left of the bit-identity threshold ($10^{-5}$, dashed), so the reported $1.000000$ is exact equality realized in floating point, not an artifact of rounding or a lenient threshold.}
\label{fig:identity}
\end{figure}

\begin{figure}[t]
\centering
\includegraphics[width=0.7\columnwidth]{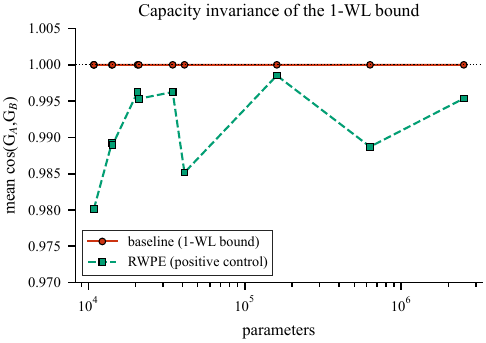}
\caption{Cosine similarity between the two embeddings of a pair as the hierarchical model grows from $10^4$ to $2.5{\times}10^6$ parameters. The baseline stays pinned at exactly $1.0$ at every capacity, so larger models do not approach the boundary; the \rwpe{}-augmented control stays strictly below $1.0$ throughout. Expressiveness here is set by the input encoding, not by model size.}
\label{fig:capacity}
\end{figure}

\section{Additional Empirical Figures}
\label{app:figs}
This appendix gathers the remaining figures behind the quantitative claims of \Cref{sec:main,sec:probes,sec:rwpe,sec:discussion}. \Cref{fig:main} shows the full encoder$\times$transform grid: at baseline and under the virtual-node and LP-derived transforms every cell is exactly $1.0$ (bit-identical), and only the structural input encodings drop the similarity, the visual counterpart of \Cref{tab:main}. \Cref{fig:probes} plots the probe battery of \Cref{sec:probes}, contrasting recoverable targets on random MILPs with their non-recoverability on the 1-WL-equivalent population. \Cref{fig:plateau} shows the joint-training experiment of \Cref{sec:rwpe}: the baseline loss is pinned at $\ln 2$ for the entire run while the \rwpe{}-augmented model converges toward zero, indicating that the equivalence is determined by the architecture rather than by optimization. \Cref{fig:footprint} summarizes the practical footprint discussed in \Cref{sec:discussion} in three panels: the escape size of each input encoding (\Cref{fig:encoding}), the downstream connectivity-prediction accuracy under the baseline encoding and under \rwpe{} (\Cref{fig:downstream}), and the prevalence of WL-twin nodes in random standard-family corpora (\Cref{fig:prevalence}). Whole-instance 1-WL collisions are rare (none among $7{,}140$ same-shape pairs per family), but WL-twin \emph{nodes} within an instance are common: $12.0\%$ of random set-cover and $6.3\%$ of combinatorial-auction instances contain them.

\begin{figure}[t]
\centering
\includegraphics[width=\columnwidth]{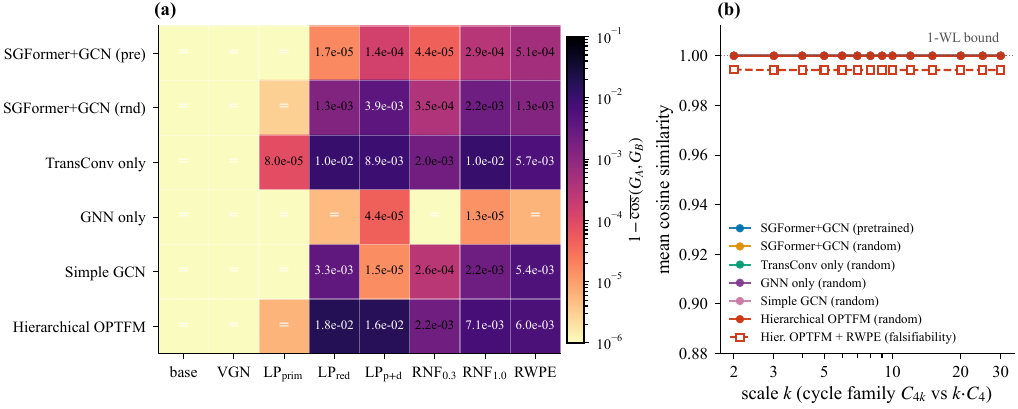}
\caption{Cosine similarity between the two embeddings of each pair, by encoder (rows) and input transform (columns). At baseline, and under the virtual-node and LP-derived transforms, every cell is exactly $1.0$: the two embeddings are bit-identical. Only changing the \emph{input encoding} to a structural one (\rwpe{}, \rni{}, \lappe{}) drops the similarity below $1.0$, consistent with \Cref{tab:main} and with the localization of expressiveness to the input encoding (\Cref{sec:rwpe}).}
\label{fig:main}
\end{figure}

\begin{figure}[t]
\centering
\includegraphics[width=\linewidth]{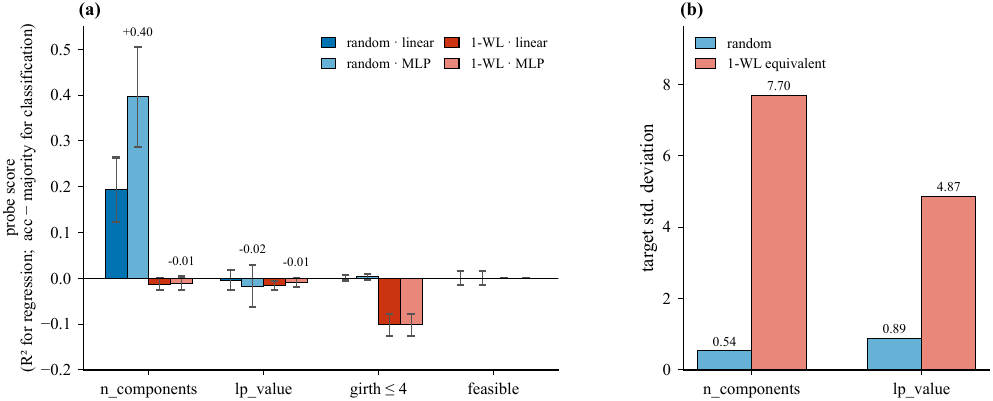}
\caption{Probe $R^2$ on random MILPs (where targets are recoverable) versus the 1-WL-equivalent population (where they are not recoverable), for structural targets not computed by 1-WL. The effect persists despite larger target variance on the 1-WL-equivalent side.}
\label{fig:probes}
\end{figure}

\begin{figure}[t]
\centering
\includegraphics[width=0.7\columnwidth]{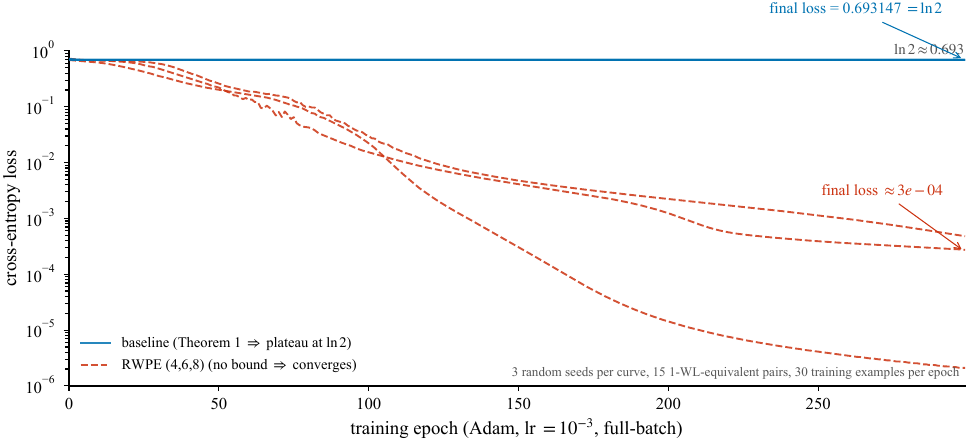}
\caption{Joint training to separate the pairs. The baseline loss is pinned at $\ln 2$ (chance) for the entire run, while the same model trained on \rwpe{}-augmented inputs converges toward zero. An architectural ceiling is invariant to training; a structural input encoding lifts it.}
\label{fig:plateau}
\end{figure}

\begin{figure}[t]
\centering
\begin{minipage}[b]{0.32\linewidth}
  \centering
  \includegraphics[width=\linewidth]{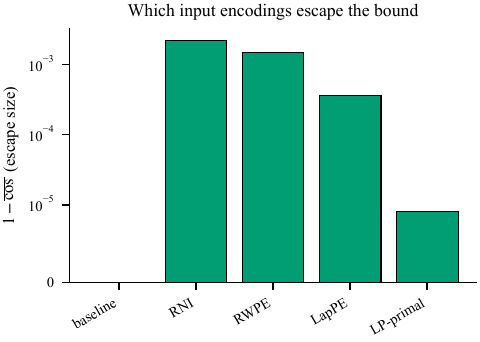}
  \subcaption{Escape size per input encoding ($1-\overline{\cos}$, log scale).}
  \label{fig:encoding}
\end{minipage}
\hfill
\begin{minipage}[b]{0.32\linewidth}
  \centering
  \includegraphics[width=\linewidth]{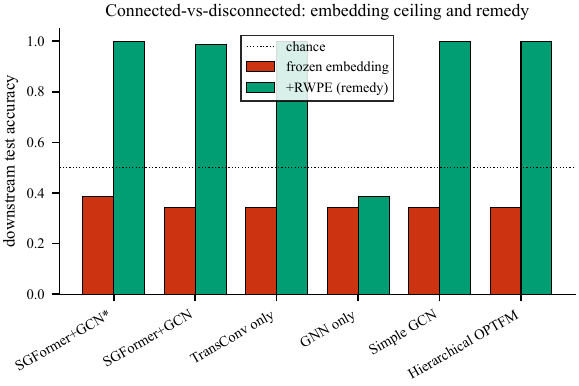}
  \subcaption{Connectivity-prediction accuracy: baseline vs.\ \rwpe{}.}
  \label{fig:downstream}
\end{minipage}
\hfill
\begin{minipage}[b]{0.32\linewidth}
  \centering
  \includegraphics[width=\linewidth]{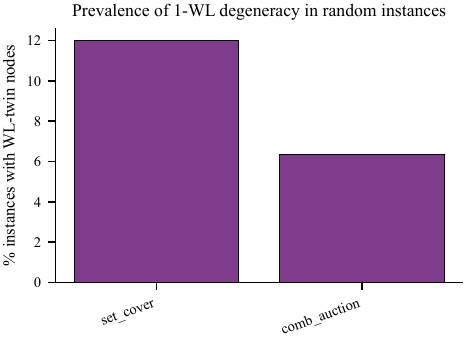}
  \subcaption{WL-twin nodes in random standard-family instances.}
  \label{fig:prevalence}
\end{minipage}
\caption{The bound's practical footprint, in three views.
\textbf{(a)}~(\subref{fig:encoding}) How far each input encoding pushes the two embeddings apart: the structural encodings (\rni{}, \rwpe{}, \lappe{}) separate the pairs by two to three orders of magnitude more than an LP-derived feature, which barely moves them because the LP optimum often coincides, motivating \rwpe{} as the primary remedy. \textbf{(b)}~(\subref{fig:downstream}) The downstream cost: an embedding-only classifier of a structural label is pinned near chance for every encoder at baseline, and an \rwpe{} input lifts most to near-perfect accuracy. \textbf{(c)}~(\subref{fig:prevalence}) The degeneracy is not only adversarial: $12\%$ of random set-cover and $6\%$ of combinatorial-auction instances already contain WL-twin nodes, so reduced resolution is a measurable property of ordinary corpora.}
\label{fig:footprint}
\end{figure}

\section{Full Probe Results}
\label{app:probes}
This appendix expands the probing analysis of \Cref{sec:probes}. \Cref{tab:probes} reports probe scores for all five graph-level targets on both populations, with each target's variance and the \rwpe{} capacity control. Non-recoverability on the 1-WL-equivalent population holds for every structural target and persists despite equal or larger target variance, while the capacity control recovers the signal in each case, indicating that the effect is a property of the embedding rather than of probe capacity.

\begin{table}[t]
\centering
\caption{Probe scores (MLP, 5 seeds): $R^2$ for regression targets, accuracy for classification. Non-recoverability on the 1-WL-equivalent population holds \emph{despite larger target variance}, and the capacity control (the same probe on \rwpe{}-augmented embeddings of the same instances) recovers the signal, indicating the effect is a property of the embedding rather than of probe capacity. Values are mean\,$\pm$\,std over five seeds.}
\label{tab:probes}
\small
\begin{tabular}{@{}lrrrr@{}}
\toprule
Target & var(rnd) & var(1wl) & rnd & 1wl \,/\, +\rwpe{} \\
\midrule
n\_components & 0.22 & 26.37 & $+0.23{\scriptstyle\pm.04}$ & $-0.01{\scriptstyle\pm.02}$ \,/\, $+0.51{\scriptstyle\pm.05}$ \\
num\_4cycles & 32.48 & 30.33 & $+0.58{\scriptstyle\pm.03}$ & $-0.01{\scriptstyle\pm.02}$ \,/\, $+0.58{\scriptstyle\pm.04}$ \\
spectral\_gap & 0.025 & 0.016 & $+0.50{\scriptstyle\pm.06}$ & $-0.00{\scriptstyle\pm.03}$ \,/\, $+0.20{\scriptstyle\pm.05}$ \\
lp\_value & 1.00 & 14.03 & $+0.00{\scriptstyle\pm.03}$ & $-0.02{\scriptstyle\pm.02}$ \,/\, $+0.13{\scriptstyle\pm.04}$ \\
girth$\le$4 (acc) & 0.02 & 0.25 & $0.98{\scriptstyle\pm.01}$ & $0.44{\scriptstyle\pm.05}$ \,/\, $1.00{\scriptstyle\pm.00}$ \\
\bottomrule
\end{tabular}
\vspace{2pt}
\end{table}

\section{Experimental Details}
\label{app:expdetails}
\paragraph{Encoders.}
Six foundation-family encoders: the pre-trained \sgformer{}+GCN checkpoint and its random-initialized twin; an attention-only ablation; a GCN-only ablation; a minimal GCN; and the full hierarchical multi-view model. Four independent encoders: a Graphormer-style model with degree and adjacency encodings~\citep{ying2021graphormer}; a GraphGPS-style hybrid~\citep{rampasek2022gps}; a Set-Transformer pooling model~\citep{lee2019set}; and the original Gasse bipartite GCN~\citep{gasse2019exact}. Because \Cref{thm:main} holds for all weights, random initialization suffices; where a pre-trained checkpoint exists, it is used and confirmed to behave identically.

\paragraph{Input transformations and metrics.}
Each encoder is evaluated under the unmodified baseline; the virtual-global-node augmentation; random node initialization (\rni{})~\citep{abboud2020rni}; deterministic LP-derived features; random-walk positional encodings (\rwpe{})~\citep{dwivedi2022rwpe}; and Laplacian positional encodings (\lappe{})~\citep{dwivedi2021lap}. For each (encoder, transform, pair), we report mean cosine similarity, the worst-case coordinate gap $\|\Phi(M_A)-\Phi(M_B)\|_\infty$, and the exact-match rate (EM). Randomized settings use five seeds with bootstrap $95\%$ confidence intervals. All experiments run on a laptop CPU in minutes; no GPU is required.

\paragraph{LP-derived features.}
The LP-derived transform solves the continuous relaxation of the MILP once ($\min\{c^\top x: Ax\circ b,\ l\le x\le u\}$, dropping the integrality constraints) and appends per-node features read off from the optimal primal--dual solution: for each variable node, its primal LP value $x^\star_j$, reduced cost, and a tight-bound indicator; for each constraint node, its dual value $y^\star_i$ and slack. These are the same solver-derived quantities used in learning-to-branch pipelines~\citep{gasse2019exact,prouvost2020ecole}. Because our constructed pairs are built from regular or vertex-transitive topologies with symmetric objectives, the LP relaxation attains the same optimal value with a symmetric optimal solution on $G_A$ and $G_B$, so the LP-derived features are themselves 1-WL-invariant on these pairs and leave the bound essentially intact ($\text{EM}=0.94$; \Cref{tab:enc}). LP-derived features can separate pairs whose optima break the symmetry, which is why they are reported as a weak, distribution-dependent signal rather than a general remedy.

\section{Extended Discussion of Scope and Validity}
\label{app:threats}\label{sec:threats}
\paragraph{Constructed pairs versus natural corpora.}
The exact 1-WL pairs are constructed, and exact collisions between independently sampled instances are rare; the paper does not claim that natural corpora are full of exact 1-WL twins. The claims are narrower and threefold: the bound is exact and architecture-universal, WL-twin \emph{nodes} are common in random instances, and the bound binds wherever regular or symmetric substructure occurs. A full sweep of a benchmark library with solver-derived hardness labels is a natural direction for future work; it requires a solver in the loop but does not affect the theorem.

\paragraph{Scope of the proof.}
\Cref{thm:main} covers encoders built from per-node maps and symmetric aggregation; its lemmas formalize each component of the hierarchical encoder as a symmetric multiset function and compose them. An implementation that introduces an order-dependent operation outside this class, for instance, hard top-$k$ attention, a sequence-sensitive readout, or an ordering induced by node identifiers, falls outside the theorem and would require a separate analysis. Such operations are absent from the foundation models studied here and from the standard graph transformers in \Cref{sec:related}, all of which aggregate symmetrically. The empirical study: ten encoders, multiple MILP families, capacities from $10^4$ to $2.5\times10^6$ parameters, and both single and double precision, confirms the bound on every configuration tested, including ones not literally enumerated by the lemmas.

\paragraph{Isolation of the architectural bound.}
A faithful Graphormer also uses shortest-path encodings that exceed 1-WL. These are deliberately separated from the attention mechanism and treated, like \rwpe{}, as input-level signals, so that \Cref{thm:main} isolates the contribution of the \emph{architecture} alone. This separation is what makes the localization in \Cref{sec:rwpe} precise: the expressiveness gain attributed to such transformers enters through their input features, not their attention. The study concerns expressiveness and embedding readout; evaluating an end-to-end solver trained on these embeddings is a complementary question that the theorem does not bear on.

\paragraph{Pre-trained checkpoint.}
The shipped checkpoint is for the simpler \sgformer{}+GCN baseline, not the full hierarchical model. Since the bound holds for all weights, random initialization suffices for the latter, and the pre-trained model behaves identically where it applies.

%% ================================================================
%%  12. CONCLUSION
%% ================================================================

\end{document}